\documentclass[11pt]{article}
\usepackage{amsmath,epsfig, amsthm, amssymb, enumerate}
\usepackage{fullpage}
\usepackage{graphicx}
\usepackage{subfigure}
\usepackage{amsmath, amssymb, epsfig}
\usepackage{bm}
\usepackage{mathrsfs}
\usepackage{color}

\DeclareMathOperator*{\argmin}{arg min} 
 
\DeclareMathOperator*{\supp}{supp}

\DeclareMathOperator\erf{erf}

\DeclareMathOperator{\sign}{sign}

\newcommand{\R}{\mathbb{R}}

\newcommand{\vct}[1]{\bm{#1}}
\newcommand{\mtx}[1]{\bm{#1}}

 \newcommand{\scalprod}[1]{\left\langle #1  \right \rangle}
\newtheorem{theorem}{Theorem}[]
\newtheorem{lemma}[theorem]{Lemma}

\newtheorem{cor}[theorem]{Corollary}

\newtheorem{remark}[theorem]{Remark}

\begin{document}
\title{One-bit compressive sensing with norm estimation}

\author{Karin Knudson\footnote{University of Texas at Austin, kknudson@math.utexas.edu}  \hspace{1mm}, Rayan Saab\footnote{University of California San Diego, rsaab@ucsd.edu. R. Saab has been supported in part by a UCSD  Research Committee Award, a Hellman Fellowship, and the NSF under grant DMS 1517204. }  \hspace{1mm}, and Rachel Ward\footnote{University of Texas at Austin, rward@math.utexas.edu.  R. Ward has been supported in part by an Alfred P. Sloan Research Fellowship, an AFOSR Young Investigator Award, and DOD-Navy grant N00014-12-1-0743.}}
\maketitle
  \begin{abstract}

  Consider the recovery of an unknown signal $\vct{x}$ from quantized linear measurements.  In the one-bit compressive sensing setting, one typically assumes that $\vct{x}$ is sparse, and that the measurements are of the form $\sign(\langle \vct{a}_i, \vct{x} \rangle) \in \{\pm1\}$.  Since such measurements give no information on the norm of $\vct{x}$, recovery methods typically assume that $\| \vct{x} \|_2=1$.  We show that if one allows more generally for quantized \emph{affine} measurements of the form $\sign(\langle \vct{a}_i, \vct{x} \rangle + b_i)$, and if the vectors $\vct{a}_i$ are random, an appropriate choice of the affine shifts $b_i$ allows norm recovery to be easily incorporated into existing methods for one-bit compressive sensing.  Additionally, we show that for arbitrary fixed $\vct{x}$ in the annulus $r \leq \| \vct{x} \|_2 \leq R$, one may estimate the norm $\| \vct{x} \|_2$ up to additive error $\delta$ from $m \gtrsim R^4 r^{-2} \delta^{-2}$ such binary measurements through a single evaluation of the inverse Gaussian error function.  Finally, all of our recovery guarantees can be made \emph{universal} over sparse vectors, in the sense that with high probability, one set of measurements and thresholds can successfully estimate all sparse vectors $\vct{x}$ in a Euclidean ball of known radius. 
  
\end{abstract}


\section{Introduction}
Compressive sensing, as introduced in \cite{carota06, donoho2006compressed, candes2006compressive}, concerns the approximation of a sparse (or approximately sparse) vector $\vct{x} \in \mathbb{R}^n$ from linear measurements of the form 
\[
\label{CS:basic}
y_i = \scalprod{\vct{a}_i,\vct{x}}, \quad i = 1,2, \dots, m.
\]
To allow processing using digital computers, the measurements $y_i$ must be quantized to a finite number of bits in practical compressive sensing architectures.  In the extreme case, it is of interest to consider the \emph{one-bit compressive sensing} problem, as introduced in \cite{Boufounos08},  which studies the approximation of a sparse (or almost sparse) vector $\vct{x} \in \mathbb{R}^n$ from one-bit quantized measurements of the form
\begin{equation}
\label{1bit}
y_i = \sign( \scalprod{ \vct{a}_i, \vct{x}} ), \quad i = 1,2,\dots,m,
\end{equation}
where $\sign(t) = 1$ when $t \geq 0$ and $\sign(t) = -1$ when $t < 0$.  
In  practice a comparator (one-bit quantizer) is easy to build, fast, and consumes relatively little power, so  one-bit measurements may even be \emph{preferable} in situations where finer quantization is expensive relative to additional measurements.   One-bit measurements may carry added benefits such as robustness to certain nonlinearities in the signal acquisition process (saturation, for example). Additionally, recent research indicates that in some settings, recovery from one-bit measurements may even out-perform multi-bit compressed sensing (from the point of view of total number of bits used versus reconstruction error) \cite{laska2012regime}.   We refer the reader to the webpage \cite{DSP} for a list of applications of one-bit compressed sensing. 
\subsection{Measurement Model and Objectives }We are interested in the measurement model 
\[y_i = \sign( \scalprod{ \vct{a}_i, \vct{x}}  + b_i), \quad i = 1,2,\dots,m,\]
where the vectors $\vct{a}_i$ are Gaussian random vectors, drawn once, and fixed thereafter. Our goal is to recover (effectively) sparse vectors $\vct{x}$, satisfying a norm bound, say $\|\vct{x}\|_2 \leq R$.  We consider two models for the shifts $b_i$. First we consider Gaussian random variables $b_i$, again drawn once and fixed thereafter. We also consider fixed thresholds $b_i=b$, chosen appropriately depending on a lower bound on the norm of our signals. In the latter case, if the goal is just to estimate the norm $\| \vct{x} \|_2$ and not also the direction $\vct{x}/\| \vct{x} \|_2$, our recovery method works for an arbitrary $\vct{x}$ in a fixed annulus and the sparsity assumption can be dropped.  In both cases our objective is to accurately recover the vector $\vct{x}$ (i.e., both its magnitude and direction) and for each type of threshold $b_i$ we propose a recovery technique and prove associated theoretical guarantees (in Sections \ref{sec:alga} and \ref{sec:algb}, respectively). Specifically, we prove decay bounds on the reconstruction error as the number of measurements $m$ increases.   In the remainder of the introduction, we discuss relevant prior work on one-bit compressed sensing (Section \ref{sec:Prior}), discuss our contributions (Section \ref{sec:Contributions}), and relate our methods to other quantization and reconstruction techniques (Section \ref{sec:Related}). 

\subsection{Prior Work} \label{sec:Prior}
Originally introduced in \cite{Boufounos08} by Boufounos and Baraniuk, one-bit compressed sensing was studied in detail in \cite{Jacques11} where a lower bound on the reconstruction error was provided along with heuristic algorithms for the recovery of the underlying signals. As one-bit quantization via \eqref{1bit} necessarily loses all magnitude information, the afore-mentioned bound on the reconstruction error pertained to approximating the magnitude-normalized signal $\frac{\vct{x}}{\|\vct{x}\|_2}$ by another unit-norm vector. In other words, the results were on the reconstruction accuracy associated with the direction of $\vct{x}$. 

{\bf Lower bounds.} Consider the set of bounded sparse signals $\Sigma_k^* = \{\vct{x}\in \R^n, |\supp(\vct{x})| \leq k, \|\vct{x}\|_2 \leq 1 \}$ and denote by $\vct{A}$ the $m\times n$ matrix with $\vct{a}_i$ as its rows.  
Let $\mathcal{Q} = \sign( \vct{A}\Sigma_k^*) \subset \{\pm 1 \}^m $  be the quantization of $\Sigma_k^*$  under the one-bit compressed sensing model. Thus, for each $\vct{q} \in \mathcal{Q}$ there is a quantization cell, i.e., a set of vectors $S_{{\bf q}} \subset \Sigma_k^*$ for which \[\vct{x}\in S_{{\bf q}} \implies \sign(\vct{Ax}) = \vct q.\] An optimal decoder, given $\vct{q}$, minimizes the worst case error over all $\vct{x} \in S_{{\bf q}}$, so it returns $\vct{x}^\sharp  = \argmin\limits_{ \vct{z} \in S_{{\bf q}}}\max\limits_{\vct{x}\in S_{{\bf q}}}\|\vct{x}-\vct{z}\|_2.$ 

Thus the minimal worst-case error associated with a cell is simply the radius of the cell, and the overall optimal error  $\varepsilon_{opt}$ is the radius of the largest cell. Taking this geometric view, it was shown in \cite{Jacques11} that   \[\varepsilon_{opt}\geq \frac{ck}{m+k^{3/2}} \to_m \Omega (k/m)	.\] So, at best, the error decays linearly as the number of measurements $m$ increases. This lower bound holds regardless of the reconstruction method used (whether it is numerically efficient or not) and is derived solely based on the geometry of the problem.  

{\bf Tractable recovery with theoretical guarantees.} The first  computationally tractable method (with provable error guarantees) for reconstructing effectively sparse vectors from one-bit measurements of the form \eqref{1bit}  was proposed by  Plan and Vershynin  \cite{plan2013one} (cf. \cite{Plan12}).  In particular, they prove the following:

 \begin{theorem}[Theorem 1.1 of \cite{plan2013one}] 
 \label{prop131}
Let $\vct{a}_i\in \R^n$, $i = 1,...,m$, be random vectors with independent and identically distributed standard Gaussian entries and suppose $m>C \delta^{-5}s\log^2 (2n/s)$.\footnote{In \cite{plan2013one} this bound is expressed in terms of $\delta$, but it will be more convenient for us to state in terms of $m$.} With probability exceeding $1-C\exp{(-c\delta m)}$, the following holds for every $\vct{x}\in \R^n$ with $\|\vct{x}\|_1 \leq \sqrt{s}\|\vct{x}\|_2$: \quad the solution $\vct{x}^\sharp$ to the optimization problem 
\begin{equation}
\label{1bitopt}
\min_{\vct{x}'\in\R^n} \|\vct{x}'\|_1 \quad \text{ subject to } \sum_{i=1}^m|\langle \vct{a}_i, \vct{x}'\rangle| = m \quad\text{and} \quad \sign{\langle \vct{a}_i,\vct{x}'\rangle} =\sign{\langle \vct{a}_i,\vct{x}\rangle}, \quad i \in [m]
 \end{equation}
 satisfies
 \[
 \label{known_err31}
 \left\| \frac{\vct{x}}{ \| \vct{x} \|_2} - \frac{\vct{x}^{\sharp}}{\| \vct{x}^\sharp \|_2 }\right\|_2 \leq  \delta.
  \] 
  Above, $C$ and $c$ are universal constants, independent of all other parameters.
 \end{theorem}
As alluded to earlier, a limitation of this result and in prior results treating the one-bit compressive sensing problem (e.g., \cite{Boufounos08}, \cite{Boufounos09}, \cite{jacques2011}, \cite{Plan12}, \cite{plan2013one},  \cite{Yan12})  is that the normalization $\| \vct{x} \|_2 $ must be known  \emph{a priori} to guarantee any accuracy in the reconstructed solution (we refer the reader to, e.g.,  \cite{boufounos2014quantization} for a detailed summary of prior results in the one-bit quantization setting).   If one considers only quantized linear measurements $y_i = \sign(\scalprod{\vct{a}_i, \vct{x}}$), then such an assumption \emph{must} be made:  quantized linear measurements give no information about the magnitude of the underlying vector $\vct{x}$.  
  As we will show, this problem can be resolved if one allows more generally for quantized affine linear measurements $y_i = \sign( \scalprod{\vct{a}_i, \vct{x}} + b_i).$  
 
 In certain applications, the addition of such affine shifts is natural; for example, in the application to threshold group testing  \cite{C13}, the statistician has some control over the threshold beyond which the measurement maps to a one.  Such control is also natural in the design of binary embeddings, where the goal is to find a transformation $f: \mathbb{R}^n \rightarrow \{0,1\}^m$ such that the Hamming distance between two binary codes is close to their similarity in the original space \cite{Jacques11}. 
 Of course, in certain applications it is not always possible to add fixed affine shifts. For example, if used for feature selection in classification problems, the quantization occurs naturally, (i.e., is not imposed by the user) and it is not possible to observe or design the underlying process (e.g., see \cite{Plan12} for more details).  
\bigskip
\subsection { Contributions of this paper} \label{sec:Contributions} We study the scenario where the norm of $\vct{x}$ is not known a priori, and must be estimated along with the direction, from one-bit compressive measurements.  Because measurements of the form $\sign(\scalprod{\vct{a}_i, \vct{x}})$ give no information about the norm, we consider the reconstruction of $\vct{x} \in \mathbb{R}^n$ from more general one-bit measurements of the form 
\begin{equation}
\label{1bit:new}
y_i = \sign( \scalprod{\vct{a}_i, \vct{x}} + b_i), \quad i = 1,2, \dots, m,
\end{equation}
 where $\vct{b} = (b_i)_{i=1}^m$ is known. For reconstructing $\vct{x}$ (and $\|\vct{x}\|_2$) from the measurements \eqref{1bit:new} we propose and analyze two algorithms. Our main results are presented in Theorem \ref{normest1}, Theorem \ref{normest2}, Theorem \ref{normest_uniform} and Corollary \ref{combined}. Below, we discuss both algorithms: 

\begin{itemize}
\item {\bf \emph{Augmented convex programming approach.}} When the shifts $b_i$ are standard Gaussian variables and the measurement vectors $\vct{a}_i$ have i.i.d. standard Gaussian entries, we can re-write the affine measurements \eqref{1bit:new} as augmented linear measurements
\begin{align}
y_i &= \sign( \scalprod{\vct{a}_i, \vct{x}} + b_i) =  \sign( \scalprod{\tilde{\vct{a}_i}, \tilde{\vct{x}}}), \nonumber 
\end{align}
where $\tilde{\vct{a}}_i = (\vct{a}_i, b_i)$ and $\tilde{\vct{x}}\in \R^{n+1}$ is given by  $\tilde{\vct{x}}=(\vct{x}, 1)$.  We use a standard one-bit compressed sensing recovery method, such as that of Plan and Vershynin \eqref{1bitopt} to give us an estimate $\vct{x}^{\sharp}$ of $\tilde{\vct{x}}$, albeit without magnitude information. Defining $[n]:=\{1,...,n\}$ and denoting by $\vct{x}_T$ the restriction of $\vct{x}$ to $T\subset [n]$, we note that if each ratio $ x^{\sharp}_j /\tilde{x}_j$ is roughly the same then the ratio of the norms $ \| \vct{x}_{[n]}^{\sharp} \|_2 / \| \vct{x} \|_2 $ should be close to the \emph{known} ratio $x^{\sharp}_{n+1} / \tilde{x}_{n+1} = x^{\sharp}_{n+1}/1.$  Rearranging, this gives
$$\| \vct{x} \|_2  \approx  \frac{ \| \vct{x}^{\sharp}_{{[n]}}\|_2 }{x^{\sharp}_{n+1}}.$$ 
In Section \ref{sec:alga}, we formalize this intuition and prove theoretical guarantees for this method (see Theorem \ref{normest1}). 

\item {\bf \emph{Empirical distribution function approach.}} When the shifts $b_i$ are all set to a common, non-random value $\tau$, we propose a method based on the empirical cumulative distribution function to estimate the norm of $\vct{x}$.    This method is motivated by the observation that $\langle \vct{a_i}, \vct{x} \rangle$ is a Gaussian random variable with mean zero and standard deviation of $\|\vct{x}\|_2$, which is the quantity we wish to estimate.  Thus, the fraction of the measurements  $\sign( \scalprod{\vct{a}_i, \vct{x}} - \tau)$ that are negative should approximate the cumulative distribution function of a ${\cal N}(0, \| \vct{x}\|^2)$ random variable evaluated at $\tau$.  The accuracy of the empirical cumulative distribution function (empirical cdf or EDF) is guaranteed by the Dvoretzky-Kiefer-Wolfowitz (DKW) Inequality from \cite{Dvoretzky1956}, and we use the value of the empirical cdf at $\tau$ to obtain an estimate for $\| \vct{x} \|_2$. In fact, for norm estimation alone, our results hold for an arbitrary fixed $\vct{x}$ in the annulus $r \leq \| \vct{x} \|_2 \leq R$.  Specifically, Theorem 10 says that one may estimate the norm $\| \vct{x} \|_2$ of such an $\vct{x}$ up to additive error $\delta$ from $m \gtrsim R^4 r^{-2} \delta^{-2}$ quantized binary measurements through a \emph{single evaluation of the inverse Gaussian error function}, serving as a Johnson-Lindenstrauss type embedding \cite{JL, Dasgupta} for binary measurements.  Theorem 11 strengthens this result for the class of sparse signals in the annulus by providing a \emph{uniform} error guarantee. Corollary 12 shows how Theorem 10 (or Theorem 11) can be combined with standard one-bit recovery methods (e.g., Theorem \ref{prop131}) to estimate both the norm and direction together.  Section \ref{sec:algb} presents our theoretical results on this method. 

\end{itemize}

  Both methods assume a known upper bound on the norm of $\vct{x}$ and the EDF method further assumes a known lower bound on $\| \vct{x} \|_2$.  For each method we present sufficient conditions on $m$ for \emph{universal} sparse signal recovery to hold with high probability (to within a desired accuracy $\delta > 0$).  We show that the performance of the augmented convex programming approach scales like $\| {\vct x} - {\vct x}^{\sharp} \|_2 \lesssim 1/m^{1/5}$, similar to the theoretical rate given in \cite{plan2013one} in the case where $\| \vct{x} \|_2= 1$ is assumed.   We show that the EDF method is guaranteed to do at least this well, and in certain regimes even achieves the scaling $\| {\vct x} - {\vct x}^{\sharp} \|_2 \lesssim 1/m^{1/2}$.  
  
  We include numerical experiments comparing the accuracy of each norm recovery method, and find that empirically, the performance of both methods scales like $\| {\vct x} - {\vct x}^{\sharp} \|_2 \lesssim 1/m$,  matching the known lower bound for the performance for one-bit compressive sensing {\cite{jacques2011}.  The numerical experiments suggest that the EDF method is more sensitive to the choice of parameters such as the lower and upper bounds on $\| \vct{x} \|_2$.  At the same time, for norm estimation alone, the EDF method is much more computationally efficient than solving a convex program. For example, the cost of the convex problem \eqref{1bitoptnorm} grows polynomially in the dimensions of the problem whereas the EDF method requires only a single evaluation of the inverse Gaussian error function, once the fraction of measurements quantized to $-1$ is tallied. It is thus linear in the number of measurements and does not even require knowledge of the measurement matrix $\vct{A}$.
  
Finally, we note that the proposed EDF method uses constant affine shifts $b_i = b,$ and the addition of such shifts should not incur any additional difficulties in the one-bit hardware design.   However, our theoretical results for the convex programming approach (Theorem \ref{normest1}) rely on the affine shifts being independently randomly generated.  The randomness in the shifts may not be necessary, and requiring them is possibly an artifact of the proof technique (as the distribution of the measurements should provide sufficient variability to recover the norm with a fixed dither). Nevertheless, from a practical point of view, we comment that one need only generate the shifts once as our result holds uniformly for all signals $\vct{x}$, with high probability on the draw of the vectors $\vct{a}_i$ and shifts $b_i$. Thus, when building the compressive sensors, the quantization thresholds do not need to be generated on the fly. They can be simply stored or built into the sensors. Still, it would be interesting to extend Theorem \ref{normest1} so it also holds for constant shifts.
  
  \subsection{Related work} \label{sec:Related}
The effectiveness of introducing random dither into the one-bit quantization framework is well-established (e.g., \cite{dk06, rac03}), albeit in \emph{non compressed sensing settings}. Moreover, the application of acquisition noise prior to quantization was recently shown to enable accurate reconstruction in the context of one-bit matrix completion \cite{davenport2012}, but towards a different purpose.  Additionally, the paper \cite{bbr13} demonstrated the strong robustness of one-bit compressive sensing to random noise added pre-quantization.  

The concept of estimating a signal, \emph{including its magnitude}, by changing the threshold of a one-bit quantizer \emph{adaptively} or by dithering is also well established. The vast literature on one-bit Sigma-Delta quantization  studies how adaptivity in the threshold selection can yield reconstruction errors  that decay quickly as a function of the number of measurements or, depending on the setting, as a function of the oversampling rate. We refer the reader to, e.g., \cite{Daub-dev, G-exp, DGK10}, for results in the setting of band limited functions  and to, e.g., \cite{benedetto2006sigma, blum:sdf, KSW12, KSY13} for results in the finite-frames setting. Recently\footnote{In particular, after the initial submission of this manuscript.}, one-bit Sigma-Delta quantization has also been shown to be effective in the compressed sensing context \cite{saab2015quantization}. Specifically, \cite{saab2015quantization} shows that certain one-bit (and multi-bit) Sigma-Delta quantization schemes achieve polynomial (and  root-exponential) error decay as a function of the number of measurements. Other recent work in the compressed sensing setting (e.g., \cite{Kamilov, baraniuk2014exponential}) also uses adaptive threshold selection for one-bit quantization and \cite{baraniuk2014exponential} achieves exponential decay  of the reconstruction error as a function of the number of measurements. 
These different methods for one-bit quantization (including our own) provide different trade-offs of computational and implementation complexity against reconstruction accuracy. 
For example, the Sigma-Delta approach requires memory elements to store certain state-variables (related to the thresholds) and it sequentially quantizes incoming measurements. It requires $r$ memory elements to achieve a reconstruction error decay of $O(m^{-r})$. Moreover, if one is allowed to choose the optimal $r$ as a function of $m$, then one can even obtain reconstruction error decay of $O(e^{-c\sqrt{m}})$. The scheme of \cite{baraniuk2014exponential} achieves error decay rates of $O(e^{-cm})$, but it requires a polynomial time algorithm (in the ambient dimension $n$) to update the quantization thresholds.  Thus, it requires the sensors collecting the measurements to have significant computational power and it incurs delays in acquiring the measurements (while the thresholds are updated). Moreover, with the approach in \cite{baraniuk2014exponential}, the quantization thresholds themselves must be transmitted as they are needed for reconstruction. In contrast, the non-adaptive one-bit scheme that we study in this paper is more simple (hence easier to implement) than either of the above approaches. It simply compares each incoming measurement to a fixed (known) threshold and yields a reconstruction error decay of $O(m^{-1/5})$, see Theorem \ref{normest1}. We conclude this discussion by reiterating that the choice of which one-bit (or even multi-bit)  scheme to use for quantizing compressed sensing measurements depends on the computational power, hence implementation complexity and cost, that one is willing to expend at the sensor. It also depends on whether one has many sensors collecting  spatial data at one time (making adaptive threshold selection difficult), or one sensor collecting data temporally (thereby enabling adaptive threshold selection).

  \section{Preliminaries}
  Throughout, we use $C, c, C_1,$ etc. to denote absolute constants whose values may change from line to line.  For integer $n$ we denote $[n] = \{1,2, \dots, n\}$.  Vectors are written in bold italics, e.g. $\vct{x}$, and their coordinates written in plain text so that the ith component of $\vct{x}$ is $x_i$.  The $\ell_1$ and $\ell_2$ norms of a vector $\vct{x} \in \mathbb{R}^n$ are defined as $\| \vct{x} \|_1 = \sum_{i=1}^n | x_i|$ and $\| \vct{x} \|_2 = ( \sum_{i=1}^n x_i^2 )^{1/2}$.  The number of nonzero coordinates of $\vct{x}$ is denoted by $\| \vct{x} \|_0 = | \text{supp}(\vct{x}) |$.  For a Gaussian random variable  $X$ with mean $\mu$ and variance $\sigma^2$, we write $X \sim {\cal N}(\mu, \sigma^2)$.

\bigskip

\noindent To prove our main results, we will need some lemmas.  The first lemma is a simple geometric inequality concerning the norm of the difference between two vectors.
 
   \begin{lemma}
  \label{sincos}
Consider vectors $\vct{x}_1, \vct{x}_2 \in \mathbb{R}^n$ and positive scalars $t_1, t_2, \alpha, \eta \in \mathbb{R}$  satisfying \\

$t_1  \geq \alpha  >  \eta, \quad  \| \vct{x}_1 \|_2^2 + t_1^2 = 1, \quad  \| \vct{x}_2 \|_2^2 + t_2^2 \leq 1, \quad  \emph{and} \quad \| \vct{x}_1 -\vct{x}_2 \|_2^2 + (t_1 - t_2)^2 \leq \eta^2 $.
Then
$$
\left \| \frac{ \vct{x}_1}{t_1} - \frac{ \vct{x}_2}{t_2} \right\|_2^2 \leq \frac{4\eta^2}{\alpha^2(\alpha  - \eta)^2}
$$
 \end{lemma}
 \begin{proof}
First, define $\varepsilon = 1- \| \vct{x}_2 \|_2^2 - t_2^2.$  By the reverse triangle inequality,
$$\sqrt{\| \vct{x}_1 \|_2^2 + t_1^2} - \sqrt{\big( \| \vct{x}_2 \|_2^2 + t_2^2} \big) \leq \sqrt{\| \vct{x}_1 - \vct{x}_2 \|_2^2 + (t_1 - t_2)^2} \leq \eta.$$
 It follows that
$1 - \sqrt{1-\varepsilon} \leq \eta$, so $(1- \eta)^2 \leq 1- \varepsilon$ and finally $\varepsilon \leq 2 \eta - \eta^2 \leq 2 \eta$.  Also by the reverse triangle inequality, $t_2 \geq t_1 - \eta \geq \alpha- \eta $.  

Next we note that
 $\| \vct{x}_1 -\vct{x}_2 \|^2 + (t_1 - t_2)^2 \leq \eta^2 $ implies
\begin{align}
	-2 \langle \vct{x}_1,\vct{x}_2 \rangle &\leq \eta^2 - (t_1-t_2)^2 - \| \vct{x}_1 \|_2^2 -  \| \vct{x}_2 \|_2^2 \nonumber \\
	&= \eta^2 - (t_1^2 + \| \vct{x}_1 \|_2^2) -  (t_2^2 + \| \vct{x}_2 \|_2^2) + 2 t_1 t_2 \nonumber \\
	&= \eta^2 - 2 + \varepsilon + 2 t_1 t_2. \nonumber 
\end{align}

Now,

\begin{align}
\left\| \frac{ \vct{x}_1}{t_1} - \frac{ \vct{x}_2}{t_2} \right\|_2^2 &=  \frac{ \|\vct{x_1}\|^2}{t_1^2} + \frac{ \|\vct{x_2}\|^2}{t_2^2} - 2 \frac{\langle \vct{x}_1, \vct{x}_2\rangle}{t_1 t_2} \nonumber \\
&\leq \frac{1-t_1^2}{t_1^2} + \frac{1-\varepsilon-t_2^2}{t_2^2} + \frac{\eta^2 - 2 + \varepsilon + 2t_1t_2}{t_1 t_2} \nonumber \\
&= \frac{(t_1-t_2)^2 -\varepsilon t_1^2 + \eta^2 t_1 t_2 +\varepsilon t_1 t_2}{t_1^2 t_2^2} \nonumber \\
&= \frac{(t_1-t_2)^2 }{t_1^2 t_2^2} + \frac{\eta^2 t_1 t_2 }{t_1^2 t_2^2} +\frac{\varepsilon t_1(t_2-t_1)}{t_1^2 t_2^2} \nonumber \\
&\leq \frac{4\eta^2 }{t_1^2 t_2^2} \nonumber
\end{align}
 where in the final inequality we use that $0 < t_1 \leq 1$, $0<t_2 \leq 1$, $t_2 - t_1 < \eta$, and $\varepsilon < 2\eta$.
 Over the range $\alpha \leq t_1 \leq 1$, and  $t_1 - \eta \leq t_2 <1$, this expression attains its maximum at $t_1 =\alpha, ~t_2 = \alpha- \eta$.  Substituting these values for $t_1$ and $t_2$ results in the bound stated in the lemma.
 
  \end{proof}

 The next lemma gives a bound on the variation of a function with an inverse dependence on the Gaussian error function.
 
 \begin{lemma}
Let $\erf: \mathbb{R} \rightarrow [-1,1]$ be the Gaussian error function, $\erf(x) = \frac{2}{\sqrt{\pi}} \int_{0}^x \exp(-t^2) dt,$ and define $h:(0,1) \to \mathbb{R}$ by 
$ h(u) =  \frac{1}{\erf^{-1}(2u-1)}$.  For $\eta > 0$ and $a,b \in \left[\frac{1}{2}+\eta, \frac{1}{2}(\erf(1) + 1) \right)$, we have $\left| h(a) -h(b) \right| \leq \left| h'(\frac{1}{2} + \eta) (b-a) \right|$.
\label{hlemma}
 \end{lemma} 

 \begin{proof}
As the derivative of the inverse error function is $\frac{d}{du} \erf^{-1}(u) = \frac{1}{2} \sqrt{\pi} \exp{\Big(\big(\erf^{-1}(u)\big) ^2}\Big)$, the derivative of $h$ is given by 
$$h'(u) =\frac{-\sqrt{\pi} \exp{\left( [\erf^{-1}(2u-1)]^2 \right)}}{[\erf^{-1}(2u-1)]^2},$$ 
which is negative and decreasing in absolute value on the interval $\left(\frac{1}{2}, \frac{1}{2}(\erf(1) + 1) \right)$.  Thus, for any $\eta >0$ and $a,b$ in $\left[\frac{1}{2}+\eta, \frac{1}{2}(\erf(1) + 1)\right)$, we have $|h(a) -h(b)| \leq |h'(\frac{1}{2}+ \eta)| |b-a|$. 
 \end{proof}
 
 \section{Main results}
 
 Here we describe and give guarantees for two methods by which the norm of an unknown vector $\vct{x} \in \mathbb{R}^n$ can estimated, possibly along with the direction, from one-bit compressive measurements $y_i = \sign( \scalprod{\vct{a}_i, \vct{x}} + b_i).$   The first method augments the convex program \eqref{prop131} to retrieve norm as well as directional information about the unknown vector, and inherits the error guarantees for that approach.  The second method estimates the norm directly from the measured proportion \\
 $\#\{i : y_i = -1\} / m$, using the shift $b_i = \tau$ and the Gaussianity of $\vct{a}_i$, to obtain a consistent estimator for $\| \vct{x} \|_2$ which is analyzed using the Dvoretzky-Kiefer-Wolfowitz inequality.  This method is very efficient to implement compared to the convex programming approach, requiring only a single evaluation of the inverse Gaussian error function.  At the same time, it is less robust to parameter uncertainty, as can be seen from numerical experiments.
 
{\subsection{Augmented convex programming}\label{sec:alga}}

\noindent Our first main result is a bound on the accuracy of approximating $\vct{x} \in \mathbb{R}^n$ with assumed structural constraint $\| \vct{x} \|_1 / \| \vct{x} \|_2 \leq \sqrt{s}$  from affine one-bit measurements
\begin{equation}
y_i = \sign( \scalprod{\vct{a}_i, \vct{x}} + b_i), \quad i = 1,2, \dots, m
\end{equation}
where $b_i$ are independent ${\cal N}(0,\tau^2)$ scalars.
For reconstruction, we consider the following augmented version of the convex program \eqref{1bitopt}:  

 \begin{align}
\label{1bitoptnorm}
\min_{\vct{z} \in \mathbb{R}^n, u \in \mathbb{R}} & \quad \|(\vct{z},u)\|_1   
~  \text{ subject to } \left\{ \begin{array}{ll} &\sum_{i=1}^m |\langle\vct{a}_i,\vct{z} \rangle + \frac{u}{\tau} b_i| =m,  \\ &
 \sign({\langle \vct{a}_i,\vct{z} \rangle}+ \frac{u}{\tau} b_i) =\sign({\langle \vct{a}_i,\vct{x}\rangle} + b_i), \quad i \in [m].
\end{array}
\right.
 \end{align}
 
The intuition is that \eqref{1bitoptnorm} is equivalent to running the optimization problem \eqref{1bitopt} on the augmented vector $(\vct{x}, \tau)$. 
Since the constraint $\| \vct{x} \|_1 / \| \vct{x} \|_2 \leq \sqrt{s}$ implies that $\| (\vct{x}, \tau) \|_1 / \| (\vct{x}, \tau) \|_2 \leq \sqrt{s+1},$ the assumptions of Theorem \ref{prop131} are in force and we can use it to show that the estimate $\tau \vct{x^\sharp}/t^\sharp$
obtained from the optimum $(\vct{x}^{\sharp}, t^{\sharp})$ of \eqref{1bitoptnorm} is sufficiently close to $\vct{x}$.

 \begin{theorem}
\label{normest1}  
Fix $\tau, R$, and $\delta > 0$ such that $\delta <  \min\{1,\tau/2 \}$. 
For $i=1,...,m$, let the random vectors $\vct{a}_i \in \R^n$ be independent and identically distributed with ${\cal N}(0,1)$ entries and let $b_i$ be  independent ${\cal N}(0,\tau^2)$ scalars.   If 
$$  m \geq C \left(\frac{\sqrt{R^2+\tau^2}}{\delta}\right)^{5} s\log^2\left( \frac{2n}{s} \right), $$
then with probability exceeding 
$1-C\exp{\left(-\frac{c\delta}{\sqrt{R^2 + \tau^2}}m\right)}$
the following holds uniformly for all vectors
$\vct{x} \in \mathbb{R}^n$ with $\| \vct{x} \|_1 \leq \sqrt{s} \| \vct{x} \|_2$ and $\|\vct{x}\|_2 \leq R$:  the estimate $\frac{\tau \vct{x^\sharp}}{t^\sharp}$
obtained from the solution $(\vct{x}^{\sharp}, t^{\sharp})$ to the optimization problem \eqref{1bitoptnorm} satisfies
$$\| \frac{\tau \vct{x^\sharp}}{t^\sharp} - \vct{x}\|_2 \leq  \frac{4\sqrt{R^2+\tau^2}}{\tau}\delta.
$$
\end{theorem}
 
 \begin{proof}
 First, observe that running the optimization problem \eqref{1bitoptnorm} with $b_i \sim {\cal N}(0, \tau^2)$ is equivalent to applying the optimization problem \eqref{1bitopt} to the augmented vector $(\vct{x}, \tau)$ with $\tilde{\vct{a}}_i \in \mathbb{R}^{n+1}$ as measurement vectors with i.i.d. standard Gaussian entries.  Since the constraint $\| \vct{x} \|_1 / \| \vct{x} \|_2 \leq \sqrt{s}$ implies that $\| (\vct{x}, \tau) \|_1 / \| (\vct{x}, \tau) \|_2 \leq \sqrt{s+1},$ we may then apply Theorem \ref{prop131} with $ m \geq C \eta^{-5} s\log^2\left( \frac{2n}{s} \right)  \geq C' \eta^{-5}(s+1) \log^2\big({\frac{2(n+1)}{s+1}}\big).$  Then
 \[
 \left\| \frac{(\vct{x},\tau)}{\sqrt{\|\vct{x}\|_2^2+\tau^2}} - \frac{(\vct{x}^\sharp, t^\sharp)}{\sqrt{\| \vct{x}^\sharp \|_2^2 + {t^\sharp} ^2 }}\right\|_2 \leq \eta
 \] with  probability exceeding $1-C\exp{(-c\eta m)}$, uniformly for all $\vct{x}$ satisfying the assumptions of the theorem. 
To finish the proof, we apply Lemma \ref{sincos}.  Since $\| \vct{x} \|_2 \leq R$ and $\delta < \tau/2$ by assumption, one easily checks that the following parameters satisfy the assumptions of Lemma \ref{sincos}: 
$\eta = \frac{\delta}{\sqrt{R^2 + \tau^2}},\\
 \alpha = \frac{\tau}{\sqrt{R^2 +\tau^2}}, \quad t_1 = \frac{\tau}{\sqrt{\| \vct{x} \|_2^2 +\tau^2}}, \quad \vct{x}_1 = \frac{\vct{x}}{\sqrt{\| \vct{x} \|_2^2 +\tau^2}}, \quad \vct{x}_2 = \frac{{\vct{x}}^\sharp}{\sqrt{\| {\vct{x}}^\sharp\|^2 + (t^{\sharp})^2}}$, and $t_2 = \frac{t^\sharp}{\sqrt{\|\vct{x}^\sharp\|^2 + (t^{\sharp})^2}}$.  Lemma \ref{sincos} gives
\begin{align*}\| \vct{x}-\tau \vct{x^\sharp}/t^\sharp \|_2^2 &= \tau^2\| \vct{x}/\tau- \vct{x^\sharp}/t^\sharp \|_2^2 \leq \tau^2 \frac{4\eta^2}{\alpha^2(\alpha - \eta)^2} \\
&= \frac{4\delta^2 (R^2 + \tau^2)}{ (\tau - \delta)^2} \leq \frac{16\delta^2 (R^2 + \tau^2)}{\tau^2} 
. 
\end{align*}
To obtain the last two inequalities above, we used the assumption $\delta < \tau/2$. 
 \end{proof}

 \noindent A few remarks are in order.

\begin{remark}[Known upper bound on $\| \vct{x} \|_2$]
If an upper bound $R$ on the norm $\| \vct{x} \|_2$ is known a priori, then one may set $\tau = R$ in the theorem to obtain the simplified error estimate $\|\vct{x}-R \vct{x^\sharp}/t^\sharp \|_2 \leq 4 \sqrt{2}\delta$.
\end{remark}
 
\begin{remark}[Tightness]
For fixed $n$, $s$, and $R$, the parameter $\lambda:=\delta^{-5}$ plays the role of an oversampling parameter and appears in the rate of decay of the reconstruction error as $\|\vct{x}-R\vct{x^\sharp}/t^\sharp \|_2 \lesssim \lambda^{-1/5}$. Compared to the known lower bound of $\|\vct{x}-\vct{x^\sharp}/t^\sharp \|_2 \gtrsim \lambda^{-1}$ for the one-bit compressive sensing problem in the case $\| \vct{x} \|_2 = 1$ and $\| \vct{x} \|_0 \leq s$, this rate is suboptimal \cite{jacques2011}.  On the other hand, this rate matches the error rate achievable using the convex optimization method \eqref{1bitopt}.

\end{remark}
 
 \begin{remark}[Alternative reconstruction methods]
The above theorem can be easily adapted to alternate reconstruction methods and inherits their associated error decay rates. For example, using the non-uniform recovery method of \cite{Plan12}, one obtains an error of $\delta$ at number of measurements $m\gtrsim \delta^{-4}R^4 s \log{n/s} $. This improves the dependence of the number of measurements on $\delta$, $R$, and $\log{n}$ at the expense of losing the uniform recovery guarantee.
\end{remark}

\subsection{Estimating $\| \vct{x} \|_2$ using the empirical distribution function}\label{sec:algb}

 In this section, we consider an alternate approach to one-bit compressive sensing with built-in norm estimation,  where now we estimate  $\| \vct{x}\|_2$ given measurements $\vct{y} = \sign({\mathbf{Ax-\boldsymbol{b}})}$ with constant (non-random) $\boldsymbol{b} = \boldsymbol{\tau} = (\tau,...,\tau) \in \mathbb{R}^m$  and $\mathbf{\tau} \neq 0$.    Unlike the previous approach, the method in this section only approximates the norm of $\vct{x}$, and gives no information about its direction. However, when combined with an estimate of $\vct{x}/\|\vct{x}\|_2$, an estimate of $\vct{x}$ can be recovered, as we show in Corollary \ref{combined}.

We consider $m$ measurement vectors $\vct{a}_i  \in \mathbb{R}^n $ whose entries $a_{i,j}$ are i.i.d. ${\cal N}(0,1)$.   Note that $\scalprod{ \vct{a}_i , \vct{x}} \sim {\cal N}(0, \|\vct{x}\|_2^2)$, and so  $\| \vct{x} \|_2$ is the standard deviation of $\scalprod{ \vct{a}_i , \vct{x}}.$  Since we only have access to the signs of the samples $\scalprod{ \vct{a}_i , \vct{x}} - \tau$, and not the samples themselves, we cannot simply estimate $\| \vct{x} \|_2$ via the sample standard deviation of $\{ \scalprod{ \vct{a}_i , \vct{x}} \}_{i=1}^m $.   Instead, we will make use of the \emph{empirical cumulative distribution function} defined by 
\begin{equation}F_m(\tau) := \frac{ \sharp \{i :y_i = -1\}}{m},\label{eq:Fm_tau}\end{equation}
which gives the proportion of the $m$ measurements $\{\scalprod{ \vct{a}_i , \vct{x}} \}_{i=1}^m $ satisfying  $\scalprod{ \vct{a}_i , \vct{x}} \leq \tau$. As $m$ increases, the random variable $F_m(\tau) $ will approach $F( \tau)= \frac{1}{2} ( 1 + \mbox{erf} (\frac {\tau}{\| \vct{x} \|_2 \sqrt{2}}) )$, where  $F$ is the cumulative distribution function of a Gaussian random variable with mean 0 and variance $\| \vct{x}\|^2_2$.  Indeed, the empirical distribution function $F_m(\tau)$ is a consistent estimator of $F(\tau)$.  We note that for $F(\tau) \neq \frac{1}{2}$, we may invert the expression for $F(\tau)$ to get $\| \vct{x}\|_2  = \frac{\tau}{\sqrt{2} \erf^{-1}(2F(\tau)-1)}$, which motivates, as an approximation of $\| \vct{x} \|_2$, the estimator 
\begin{equation}
\Lambda = \Lambda_m(\tau):=\frac{\tau}{\sqrt{2} \erf^{-1}(2F_m(\tau)-1)} \label{eq:hatR}.
\end{equation}
To help estimate the accuracy of $\Lambda$ as an approximation to $\| \vct{x} \|_2$, we turn to the Dvoretzky-Kiefer-Wolfowitz Inequality \cite{Dvoretzky1956}, which gives the following quantitative bound on the difference between a general cumulative distribution function and empirical cumulative distribution function. 

\begin{theorem}[Dvoretzky-Kiefer-Wolfowitz \cite{Dvoretzky1956}]
Let $X_1, X_2, ... , X_m$ be i.i.d. random variables with cumulative distribution function $F(\cdot)$, and let $F_m(\cdot)$ be the associated empirical cumulative density function $F_m(\tau):= \frac{1}{m}\sum 1_{X_i \leq \tau}$.  Then for any $\gamma > 0$,
$${\emph{Prob}} \hspace{.5mm} \left( \sup_\tau | F_m(\tau) -  F(\tau)| > \gamma \right) \leq  2\exp{(-2 m \gamma^2)}.$$
\label{DKW}
\end{theorem}

The DKW inequality will allow us to bound the accuracy of $\Lambda_m(\tau)$ as an estimate of $\| \vct{x} \|_2$  in Theorem \ref{normest2}. We will first need the following lemma. 
\begin{lemma}\label{lem:helper}
Fix $0<\delta <1/5$, and let $\vct{x} \in \mathbb{R}^n$ be such that $r \leq \|\vct{x}\|_2 \leq R$ for known positive constants $r$ and $R$.   Let $\mtx{A} \in \mathbb{R}^{m \times n}$ be a matrix with independent identically distributed ${\cal N}(0,1)$ entries.  Set $\tau=r,$ set $\boldsymbol{\tau}=( \tau, \tau,..., \tau),$ and compute $F_m(\tau)$ from  $\vct{y} = \sign (\vct{Ax} - \boldsymbol{\tau})$ via \eqref{eq:Fm_tau}.
If $$m \geq \pi e \frac{R^2}{r^2}\delta^{-2} \log{\frac{2}{\varepsilon}},$$
 then with probability at least $ 1- \varepsilon$ it holds that $$|F(\tau) - F_m(\tau) | < \frac{\delta}{\sqrt{2\pi}} \frac{r}{R}$$ and 
$$F(\tau) \text{ and } F_m(\tau) \in  \left [\frac{1}{2}\left(1+\erf\left(\frac{(1-\delta)r}{\sqrt{2}R}\right)\right) ~,~ \frac{1}{2}\Big(1+\erf\left(1\right)\Big)\right ].$$

\end{lemma}
\proof
By Theorem \ref{DKW}, we have for any choice of $\gamma>0$ , that $|F(\tau) - F_m(\tau)| \leq \gamma$ with probability at least $1-2 \exp{(-2m\gamma^2)}$.  Set $\tau =r $ and note that 
\begin{equation}F(\tau) = \frac{1}{2}\left( 1 + \mbox{erf} \left( \frac{\tau}{\| \vct{x} \|_2\sqrt{2} } \right) \right) \in  \left [\frac{1}{2}\left(1+\erf\left(\frac{r}{\sqrt{2}R}\right)\right) ~,~ \frac{1}{2}\Big(1+\erf \left(\frac{1}{\sqrt{2}}\right)\Big)\right ]. \label{eq:Ftau}\end{equation}
For $\delta<1$, set  $$\eta = \frac{1}{2}\erf\left(\frac{(1-\delta) r}{\sqrt{2}R}\right)$$
 and
  $$\gamma = \frac{1}{2}\left(\erf\left(\frac{r}{\sqrt{2}R}\right) -\erf\left(\frac{(1-\delta)r}{\sqrt{2}R}\right)  \right).$$
Noting that 
\begin{equation}\frac{d\erf(x)}{dx} = \frac{2}{\sqrt{\pi}}\exp{(-x^2)}\label{eq:deriv_erf},\end{equation}
we have for $0\leq a\leq b$ 
\begin{equation}  (b-a)\frac{2}{\sqrt{\pi}}\exp{(-b^2)} \leq \erf(b)-\erf(a) \leq (b-a)\frac{2}{\sqrt{\pi}}\exp{(-a^2)}.\label{eq:erf_bounds} \end{equation}
Consequently 
$$ \frac{\delta}{\sqrt{2\pi e}}\frac{r}{R} \leq \gamma  \leq 
 \frac{\delta}{\sqrt{2\pi}}\frac{r}{R}.$$
By the DKW inequality, with probability exceeding \[1-2\exp(-2m\gamma^2 ) \geq 1-2\exp\left(-\frac{\delta^2 r^2}{\pi e R^2} m\right),  \] we have $|F_m(\tau)-F(\tau)| \leq \gamma$. Together with \eqref{eq:Ftau} this gives
\[F_m(\tau) \in \left[ \frac{1}{2}+\eta,~~ \frac{1}{2}+\frac{1}{2}\erf\left(\frac{1}{\sqrt{2}}\right) + \frac{1}{2}\left(\erf\left(\frac{r}{\sqrt{2}R}\right) -\erf\left(\frac{(1-\delta)r}{\sqrt{2}R}\right)  \right)  \, \right ]. 
\]
This yields the conclusion of the lemma provided \[ 2\gamma=\erf\left(\frac{r}{\sqrt{2}R}\right) -\erf\left(\frac{(1-\delta)r}{\sqrt{2}R}\right)   \leq \erf\left(1\right)-\erf\left(\frac{1}{\sqrt{2}}\right), \]
which holds when $\delta\leq 1/5$, as then we have $\gamma \leq \frac{\delta r}{\sqrt{2\pi} R} \leq \frac{1}{5\sqrt{2\pi}} \leq \frac{1}{2}( \erf\left(1\right)-\erf\left(\frac{1}{\sqrt{2}}\right)$).

\begin{theorem}
Fix $0<\delta <\frac{2\sqrt{e}}{5}R$, and let $\vct{x} \in \mathbb{R}^n$ be such that $r \leq \|\vct{x}\|_2 \leq R$ for known strictly positive constants $r$ and $R$.   Let $\mtx{A} \in \mathbb{R}^{m \times n}$ be a matrix with independent identically distributed ${\cal N}(0,1)$ entries.  Set $\tau=r$ and compute $F_m(\tau)$ and $\Lambda = \Lambda_m(\tau)$  from  $\vct{y} = \sign (\vct{Ax} - \boldsymbol{\tau})$ via \eqref{eq:Fm_tau} and \eqref{eq:hatR} respectively.
If $$m \geq 4 \pi e^2 \frac{R^4}{r^2}\delta^{-2} \log{\frac{2}{\varepsilon}},$$
 then with probability at least $ 1- \varepsilon$ it holds that $$\Big|\| \vct{x} \|_2 -\Lambda_m(\tau)\Big| \leq \delta.$$
\label{normest2}
\end{theorem}

\begin{proof}
Define the function $h:(0,1) \to \mathbb{R}$ as in Lemma \ref{hlemma} by $h(u) = \frac{1}{\erf^{-1}(2u-1)}$. 
Then $|\|\vct{x}\|_2 - \Lambda_m(\tau)| = \frac{\tau}{\sqrt{2}} \left| h(F(\tau)) - h(F_m(\tau)) \right|$.    If $F(\tau)$ and $F_m(\tau)$ are in $\left[1/2+\eta, \frac{1}{2}(\mbox{erf}(1) + 1) \right)$ for some $\eta>0$, then by Lemma \ref{hlemma} we have: 
$$|\|\vct{x}\|_2 - \Lambda_m(\tau)| = \frac{\tau}{\sqrt{2}} |h(F(\tau)) - h(F_m(\tau))| \leq  \frac{\tau}{\sqrt{2}} |h'(1/2 + \eta)| |F(\tau)-F_m(\tau)|.$$     
Indeed, provided $\delta_0 := \frac{1}{2R\sqrt{e}} \delta <1/5$, by Lemma \ref{lem:helper} we have that $|F(\tau) - F_m(\tau)| \leq \gamma:=\frac{\delta_0}{\sqrt{2\pi}}\frac{r}{R}$, and that $F_m(\tau)$ and $F(\tau)$ do satisfy the assumptions of Lemma \ref{hlemma} with probability at least $ 1-\varepsilon$, for $\eta = \frac{1}{2} \erf(\frac{(1-\delta_0)r}{\sqrt{2}R}).$ 
   So using Lemma \ref{hlemma} and the definition of $\gamma$, and recalling $\tau = r,$ we conclude that
\begin{eqnarray*}
 |\| \vct{x}\|_2 - \Lambda_m(\tau)| &&= \frac{\tau}{\sqrt{2}} |h(F(\tau)) - h(F_m(\tau))|\\
   &&\leq \frac{\tau}{\sqrt{2}} | h'(1/2 + \eta) | \frac{\delta_0}{\sqrt{2\pi}}\frac{r}{R}
  \\
  &&=\frac{\tau}{\sqrt{2}} \sqrt{\pi} ( h(1/2 + \eta) )^2 \exp( (1/h(1/2+\eta))^2) \frac{\delta_0}{\sqrt{2\pi}}\frac{r}{R} \\
  &&= \frac{\tau}{\sqrt{2}} \sqrt{\pi} \left( \frac{1}{\erf^{-1}\left( \erf(\frac{(1-\delta_0)r}{\sqrt{2}R})\right)} \right)^2 \exp\left( \left( \erf^{-1}\left( \erf(\frac{(1-\delta_0)r}{\sqrt{2}R})\right) \right)^2  \right) \frac{\delta_0}{\sqrt{2\pi}}\frac{r}{R} \\
    &&= \frac{\tau}{\sqrt{2}} \sqrt{\pi} \left( \frac{\sqrt{2}R}{(1-\delta_0)r} \right)^2 \exp\left( \left( \frac{(1-\delta_0)r}{\sqrt{2}R} \right)^2  \right) \frac{\delta_0}{\sqrt{2\pi}}\frac{r}{R} \\
  &&\leq 2R \sqrt{e} \delta_0\\
  && = \delta.
  \end{eqnarray*}
The second-to-last equality uses the identity $\erf^{-1}(\erf(x)) = x$, and the final inequality uses that $\delta_0 < 1/5$ by assumption.
\end{proof}

The previous theorem gave a bound for norm estimation for a particular fixed $\vct{x}$, and we assumed no particular structural constraints on $\vct{x}$.   We now provide a \emph{universal} norm estimation bound akin to Theorem \ref{normest1} for the class of $s$-sparse vectors: $\vct{x} \in \mathbb{R}^n$ satisfying $\| \vct{x} \|_0 \leq s$, $r \leq \| \vct{x} \|_2 \leq R$.

\begin{theorem}
For $ i = 1, \dots, m$, let the random vectors $\vct{a}_i \in \mathbb{R}^n$ have i.i.d. ${\cal{N}}(0,1)$ entries.  
Fix the positive constants $r \leq R$ and $0<\delta \leq R$. 
There exists a constant $C_1$, such that if 
$$
m\geq C_1\frac{R^4}{r^2}\delta^{-2} s \log\left(\frac{nR^2}{s\delta r}\right)
$$ 
then with probability exceeding $1  - 14\exp\left(-\frac{\delta^2 r^2}{C_1R^4} m\right)$, the bound
$$
\Big| \| \vct{x} \|_2 -\Lambda_m\left(\frac{3r}{5} \right) \Big| \leq \delta
$$
holds uniformly for all vectors $\vct{x} \in \mathbb{R}^n$ in the set $\{ r \leq \| \vct{x} \|_2 \leq R\} \cap \{\| \vct{x} \|_0 \leq s\}$.
Here, $\Lambda_m(\frac{3r}{5})$ is the estimator computed from  $\vct{y} = \sign (\vct{A}\vct{x} - \boldsymbol{\tau})$ via \eqref{eq:hatR} with $\tau = \frac{3r}{5}$.
\label{normest_uniform}
\end{theorem}

\begin{proof}
The idea of the proof is to first show that the conclusions of Theorem \ref{normest2} hold uniformly over a sufficiently fine net of points contained in the set of $s$-sparse vectors of bounded norm. We then show that the EDF associated with an arbitrary $\vct{x}$ is well approximated by the EDF associated with some element of the net. This will allow us, via the function $h$  (from Lemma \ref{hlemma}), to obtain a bound on the norm estimation error (which holds uniformly for the class of bounded $s$-sparse vectors).

It will be helpful below to define $F_m$ and $\Lambda$ as functions of more than one argument, so $F_m(\tau, \vct{z}) := \frac{ \# \{ i: \langle \vct{a}_i, \vct{z} \rangle < \tau  \} }{m}$ and $\Lambda_m(\tau, \vct{z}) := \frac{\tau}{\sqrt{2} \text{erf}^{-1}(2 F_m(\tau, \vct{z})-1)}$. Moreover, we will require the radial projector 
$$P(\vct{q}) = \frac{\vct{q}}{\| \vct{q} \|_2} \cdot \max\{ \frac{5r}{3}, \| \vct{q} \|_2 \}, $$ and the set of bounded sparse vectors ${\cal S}:=  \{\vct{x} \in \mathbb{R}^n: \| \vct{x} \|_2 \leq R, \| \vct{x} \|_0 \leq s \}.$

 \noindent
{\bf Step (I)} Our first goal is to prove that for a finite subset of points ${\cal Q} \subseteq \cal{S}$ satisfying
\begin{equation}
\label{eq:proximity}
\min_{\vct{q} \in {\cal Q}} \| \vct{x} - \vct{q} \|_2 \leq \frac{\xi}{8} \quad \quad \text{ for each }\vct{x} \in {\cal S},
\end{equation}
we have
\begin{equation}
\label{eq:proximity2}
 \min_{\vct{q} = P(\vct{q}'): \vct{q}' \in \mathcal{Q}} \| \vct{x} - \vct{q} \|_2 \leq \xi/4 \quad \text{ for each } \vct{x} \in {\cal S} \text{ with }  \| \vct{x} \|_2 \geq \frac{5r}{3}.
\end{equation}
Along the way, we will control the cardinality of $\mathcal{Q}$.

In fact, by a well-known result in the literature on covering numbers (see, e.g.,\cite{foucart2013mathematical}[ Appendix C.2]) a set $\mathcal{Q}$ as in \eqref{eq:proximity} exists. Let $B_2^n$ denote the unit Euclidean ball in $\mathbb{R}^n$.   Given a fixed $s$-dimensional linear subspace $T$ of $\mathbb{R}^n$, there exists a finite set of points ${\cal Q}_T$ in $B_2^n \cap T$ such that $\max\limits_{\vct{x} \in B_2^n \cap T} \min\limits_{\vct{q} \in {\cal Q}_T} \| \vct{x} - \vct{q} \|_2 \leq \frac{\xi}{8}$, and  $|{\cal Q}_T| \leq (24 /\xi)^s.$  Picking such a set of points for each of the $n\choose s$ $\leq (\frac{n e}{s})^s$ $s$-dimensional linear subspaces $T$ whose union is $\{ \vct{x} \in \mathbb{R}^n: \| \vct{x} \|_0 \leq s\}$, and rescaling, we arrive at a set of points ${\cal Q}$ in ${\cal S}$ of size

$$|{\cal Q}| \leq \left(\frac{ne}{s} \right)^s (24 R/\xi\big)^s$$ satisfying \eqref{eq:proximity}.

Now, note that by Theorem 9.2 of \cite{foucart2013mathematical} there exists a constant $C>0$ so that with probability exceeding $1-\varepsilon$ the normalized matrix $\frac{1}{\sqrt{m}} \mtx{A} \in \mathbb{R}^{m \times n}$ has the \emph{restricted isometry property} 
 of order $2s$ at level $\delta$ \cite{CRT05}, provided $m>C\delta^{-2}(s\log(n/s)+\log(\frac{2}{\varepsilon}))$. That is, the normalized matrix satisfies 

$$(1-\delta) \| \vct{x} \|_2 \leq \frac{1}{\sqrt{m}} \| \mtx{Ax} \|_2 \leq (1+\delta) \| \vct{x} \|_2 \quad \quad \forall \vct{x}: \| \vct{x} \|_0 \leq 2s.$$
Henceforth we condition on the event $\mathcal{E}_1$ that ${\mtx A}$ has this property. 

 For a vector $\vct{x}  \in {\cal S}$ with $ \| \vct{x} \|_2 \geq \frac{5r}{3},$ consider the point $\vct{q}' \in {\cal Q}$ minimizing \eqref{eq:proximity}.  Since $\vct{q}'$ must then have norm $\| \vct{q}' \|_2 \geq \| \vct{x} \|_2 - \xi/8 \geq \frac{5r}{3} - \xi/8$, the triangle inequality gives
 $$
 \| \vct{x} - P(\vct{q'}) \|_2 \leq  \| \vct{x - q'} \|_2 + \|  \vct{q'} - P(\vct{q'}) \|_2  \leq \xi/8 + \xi/8 = \xi/4.
 $$
It follows that 
\begin{equation}
\label{eq:proximity2}
 \min_{\vct{q} = P(\vct{q}'): \vct{q}' \in \mathcal{Q}} \| \vct{x} - \vct{q} \|_2 \leq \xi/4 \quad \text{ for each } \vct{x} \in {\cal S} \text{ with }  \| \vct{x} \|_2 \geq \frac{5r}{3}.
\end{equation}
{\bf Step (II)} We will now show that approximating $\vct{x}$ with $\vct{q}$, as in \eqref{eq:proximity2}, entails a minor distortion in the EDF. In particular, we will show that with $\xi$ small enough  the inequalities $F_m(r/3,\vct{q}) \leq F_m(r,\vct{x}) \leq F_m(5r/3,\vct{q})$ hold.

 To that end note that $\vct{x-q}$ is $2s$-sparse since each of $\vct{x}, \vct{q} \in {\cal S}$ are individually $s$-sparse. 
Let $m^* = |\{ i: |\langle \vct{a}_i,\vct{x}-\vct{q} \rangle |^2 \geq r^2/4  \}|$. We then have
\begin{align}
m^* r^2/4 &\leq \| \mtx{A}(\vct{x-q}) \|_2^2 \nonumber \\
&\leq  {m}(1+\delta)^2 \| \vct{x-q} \|_2^2 \nonumber \\
&\leq {m}(1+\delta)^2  \xi^2 / 16. \nonumber 
\end{align}
It follows that $m^*\leq \frac{\xi^2}{r^2}m$. Moreover,
\begin{align}
m\cdot F_m(r,\vct{x}) &= | \{ i: \scalprod{\vct{a}_i,\vct{x}} \leq r \}| \nonumber\\
&\geq | \{ i: \scalprod{\vct{a}_i,\vct{q}} \leq r/2 \text{ and } \scalprod{\vct{a}_i,\vct{x-q}} \leq r/2  \}|\nonumber\\
&= m - | \{ i: \scalprod{\vct{a}_i,\vct{q}} \geq r/2 \text{ or } \scalprod{\vct{a}_i,\vct{x-q}} \geq r/2  \}|\nonumber\\
&\geq m - \big(  m (1-F_m(r/2,\vct{q})) + m^*  \big)\nonumber\\
&\geq   m \left(  F_m(r/2,\vct{q}) - \frac{\xi^2}{r^2}\right) \nonumber
\end{align}
So, repeating this calculation for the upper bound we have \[ F_m(r/2,\vct{q}) - \frac{\xi^2}{r^2} \leq F_m(r,\vct{x}) \leq F_m(3r/2,\vct{q}) + \frac{\xi^2}{r^2}.\]
We will now choose $\xi$ small enough to obtain  
\begin{equation} F_m(r/3,\vct{q}) \leq F_m(r,\vct{x}) \leq F_m(5r/3,\vct{q}). \label{eq:desired}\end{equation}
In particular, for the right hand side inequality of \eqref{eq:desired} to hold we desire
\begin{align*}
\xi^2/r^2 &\leq F_m(5r/3,\vct{q})-F_m(3r/2,\vct{q})\\
&=\big(F_m(5r/3,\vct{q}) - F(5r/3,\vct{q})\big) + \big(F(5r/3,\vct{q}) -F(3r/2,\vct{q})\big)+\big(F(3r/2,\vct{q})-F_m(3r/2,\vct{q})\big).
\end{align*}
We use \eqref{eq:Ftau} and \eqref{eq:erf_bounds} to obtain
$$
F(5r/3,\vct{q}) -F(3r/2,\vct{q}) \geq \frac{r}{ \sqrt{ 2\pi e}6R}.
$$
 To bound the first and third terms, apply Lemma \ref{lem:helper} with $\frac{\varepsilon}{(\frac{ne}{s} )^s (24 R/\xi)^s}$ in place of $\varepsilon$, separately to each $\vct{q}$.   Take a union bound over all $| {\cal Q} | \leq \left(\frac{ne}{s} \right)^s (24 R/\xi\big)^s$ elements in ${\cal Q}$ to conclude that with probability exceeding $1 - 2\varepsilon$ the event $\mathcal{E}_2$, that
\begin{equation}
\label{eq:cond2}
\max_{\vct{q} \in {\cal Q}} \left\{  \left| F_m(5r/3,\vct{q}) - F(5r/3,\vct{q}) \right|,  \quad \left| F(3r/2,\vct{q})-F_m(3r/2,\vct{q}) \right|  \right\} \leq \frac{ \delta r }{\sqrt{2\pi} R},
\end{equation}
holds once the number of measurements exceeds
\begin{equation}
\label{eq:num2}
m \geq \frac{\pi e R^2}{\delta^2 r^2}\left(\log\left(\frac{2}{\varepsilon}\right)+s\log\left(\frac{24e nR}{s \xi} \right)\right) .
\end{equation}

The calculation for the left hand inequality of \eqref{eq:desired} is very similar to the calculation above. It yields that with the same number of measurements as in \eqref{eq:num2} the event $\mathcal{E}_3$, that 
\begin{equation}
\label{eq:cond2}
\max_{\vct{q} \in {\cal Q}} \left\{  \left| F_m(r/3,\vct{q}) - F(r/3,\vct{q}) \right|,  \quad \left| F(r/2,\vct{q})-F_m(r/2,\vct{q}) \right|  \right\} \leq \frac{ \delta r }{\sqrt{2\pi} R},
\end{equation}
holds with probability exceeding $1-2\varepsilon$. Conditioning on $\mathcal{E}_2$ and $\mathcal{E}_3$, we see that \eqref{eq:desired} is achieved if 
\[ \xi^2/r^2  \leq \frac{1}{\sqrt{2\pi e}}\frac{r}{6R}  - \frac{2\delta}{\sqrt{2\pi}}\frac{r}{6R} .\]
Imposing the condition 
\begin{equation}\delta \leq \frac{1}{5}\leq  \frac{1}{3\sqrt{e}}
\label{eq:delta}
\end{equation} 
 ensures the right hand side above is greater than $\frac{1}{\sqrt{2\pi e}}\frac{r}{18R}$. Then, choosing 
\begin{equation}
\xi \leq C_2 r \sqrt{r/R}
\label{eq:xi_bound1}
\end{equation} 
with $C_2= \sqrt{ \frac{1}{18\sqrt{2\pi e}}}$ yields \eqref{eq:desired}.

{\noindent\bf Step (III)} We will now obtain the desired error bound and probability estimate. 
First, invoke Theorem \ref{normest2} with $2R\sqrt{e}\delta$ in place of $\delta$ and $\frac{\varepsilon}{(\frac{ne}{s})^s (24 R/\xi)^s}$ in place of $\varepsilon$, separately to each $\vct{q}$.  Take a union bound over all elements in ${\cal Q}$ to obtain that with probability exceeding $1 - 2\varepsilon$,
\begin{equation}
\label{eq:cond3}
\max_{\vct{q} \in {\cal Q}} \hspace{1mm} \max \{ \left| \| \vct{q} \|_2 - \Lambda_m(r/3, \vct{q}) \right|, \left| \| \vct{q} \|_2 - \Lambda_m(5r/3, \vct{q}) \right|  \} \leq 2R \sqrt{e} \delta
\end{equation}
for a number of measurements $m$ as in \eqref{eq:num2}. Condition on the event $\mathcal{E}_4$ that \eqref{eq:cond3} holds.  We have by (the proof of) Lemma \ref{hlemma} that $t \rightarrow h(t)$ is decreasing over the range $t \in [F_m(r/3,\vct{q}), F_m(5r/3,\vct{q}) ]$; it follows that
$$
\left| \| \vct{q} \|_2 - \Lambda_m(r, \vct{x}) \right| \leq \max\left\{ \left| \| \vct{q} \|_2 - \Lambda_m(r/3, \vct{q}) \right|, \left| \| \vct{q} \|_2 - \Lambda_m(5r/3, \vct{q}) \right|  \right\}
$$
and consequently
\begin{align}
\left| \| \vct{x} \|_2 - \Lambda_m(r,\vct{x}) \right| &\leq \left| \| \vct{x} \|_2 - \| \vct{q} \|_2 \right| + \left| \| \vct{q} \|_2 - \Lambda_m(r,\vct{x}) \right| \nonumber \\
&\leq  \left| \| \vct{x} \|_2 - \| \vct{q} \|_2 \right| + \max \{ \left| \| \vct{q} \|_2 - \Lambda_m(r/3, \vct{q}) \right|, \left| \| \vct{q} \|_2 - \Lambda_m(5r/3, \vct{q}) \right|  \} \nonumber \\
&\leq \xi/4  + 2R\sqrt{e}\delta.  \label{eq:whatever}
\end{align} 
Finally, setting in \eqref{eq:whatever}
 \begin{equation}
 \xi \leq 4 R\sqrt{e}\delta
 \label{eq:xi_bound2}
 \end{equation} 
 we have
\begin{align}
\left| \| \vct{x} \|_2 - \Lambda_m(r,\vct{x}) \right| &\leq 3 R \sqrt{e}  \delta.
\end{align} 
 
 To obtain the probability bound in the theorem, select $\xi = C_2 \delta r \sqrt{r/R} $ to satisfy both \eqref{eq:xi_bound1} and \eqref{eq:xi_bound2}.   Then fix $m \geq C_1 \frac{R^2}{r^2}\delta^{-2}\left(\log\left(\frac{2}{\varepsilon}\right)+s\log\left(\frac{nR}{s \delta r} \right)\right)$ for some $C_1$ large enough, as with this choice the events $\mathcal{E}_i$  ($i=1,2,3,4$) hold simultaneously,  with probability  exceeding  $1- 7\varepsilon$. Finally, set $\varepsilon = 2\exp\left(-\frac{\delta^2 r^2}{2C_1R^2} m\right)$ to get the condition $m\geq 2C_1\frac{R^2}{r^2}\delta^{-2} s \log\left(\frac{nR}{s\delta r}\right)$.
  The full statement of the theorem follows by replacing $\delta$ by $\delta/(5R)$ and $r$ by $3r/5$ (and adjusting $C_1$ accordingly). In particular, with the condition in the theorem and the substitution  of $\delta$ by $\delta/(5R)$, \eqref{eq:delta} is satisfied.
\end{proof}

As noted above, this method of norm estimation does not give us an estimate of the direction $\vct{x}$ itself; it only yields an estimate of the norm.  In order to recover $\vct{x}$, we could easily combine the estimated norm with an estimate of $\vct{x}/ \| \vct{x} \|_2$ recovered as in Proposition \ref{prop131}.  

\begin{cor}
Let $\vct{x} \in \mathbb{R}^n$ be such that $0 < r \leq \| \vct{x} \|_2 \leq R$.  Let $\delta>0$ and choose $\tau $ as in Theorem \ref{normest2}. Suppose we have $m = m_1 + m_2$ random Gaussian vectors and we collect $m_1$ measurements of the form $y_i = \sign( \langle \vct{a_i, x} \rangle - \tau)$, and $m_2$ measurements  $y_i = \sign( \langle \vct{a_i, x}) \rangle$.  Suppose $\Lambda$ is calculated from the $m_1$ measurements as in Theorem \ref{normest2} and let $\vct{x^\sharp}$ be the solution to the optimization problem in Proposition \ref{prop131}.  If $m_1 \geq C_0\delta^{-5} R^5 \left( s \log^2(\frac{n}{s}) + \log( C/\varepsilon) \right)$ and $m_2 \geq 4 \pi e^2 \frac{R^4}{r^2} \delta^{-2}\log(4/\varepsilon)$, then with probability at least $1- \varepsilon$ it holds that $ \| \Lambda \vct{x}^\sharp -  \vct{x} \|_2 \leq \delta$.
\label{combined}
\end{cor}

\begin{proof}
By Theorem \ref{prop131}, we use a convex optimization problem to obtain 
$\vct{x^\sharp}$ such that 
$$ \left\|  \vct{x}^\sharp   - \vct{x}/\| \vct{x} \|_2 \right\|_2 \leq \frac{\delta}{2}R$$
 with probability at least $ 1 - \varepsilon/2$, using only the first $m_1 \geq C\delta^{-5}R^5 ( s \log(\frac{2n}{s}))$ measurements. With the remaining $m_2 \geq C_{R,r} \delta^{-2}\log(8/\varepsilon)$ measurements, we calculate $\Lambda$ and have by Theorem \ref{normest2} that with probability at least $1-\varepsilon/2$, we have $\left| \| \vct{x} \|_2 - \Lambda \right| \leq \frac{\delta}{2}$. 
Hence, with probability at least $1 - \varepsilon \leq (1-\varepsilon/2)^2 $, we have
\begin{align}
  \| \Lambda\vct{x}^\sharp -  \vct{x} \|_2 &\leq \left\| \Lambda\vct{x}^\sharp -  \| \vct{x} \|_2 \vct{x}^\sharp \right\|_2 + \left\|\vct{x}^\sharp \| \vct{x} \|_2  - \vct{x}\right\| \nonumber \\ 
  &\leq \delta/2 + \delta \| \vct{x} \|_2 /( 2 R  ) \nonumber \\
  &\leq \delta. \nonumber
  \end{align}
  \end{proof}

\section{Numerical Experiments}

Here we test the performance of the two proposed methods for one-bit compressive sensing with norm estimation.  In all experiments, we consider  $s$-sparse vectors $\vct{x} \in \mathbb{R}^n$ with $n=300$ and $s=10$ that are constructed by a uniform draw from the set $\mathcal{S} = \{ \vct{x} : r < \| \vct{x} \|_2 < R,  \| \vct{x} \|_0 < s  \}$ for r = 10, R = 20.
We estimate $\| \vct{x} \|_2$  in two ways: (1) using the approximation  $\|\vct{\hat{x}} \|$ produced as in Theorem \ref{normest1}, and (2) using the Gaussian empirical cumulative distribution function (EDF) as in Theorem \ref{normest2} (Figures 1a, 2a).   The first estimation method is referred to as $PV_{\text{aug}}$, because it precedes by applying an augmented version of the  optimization problem \eqref{1bitopt} of Plan and Vershynin \cite{plan2013one} as in Theorem \ref{normest1}.   In a second set of experiments, we estimate $\vct{x}$ itself, rather than just its norm $\| \vct{x} \|_2$, with (1) $\vct{\hat{x}}$ as in Theorem \ref{normest1} ($PV_{aug}$), and (2) by partitioning the measurements into two sets, estimating the norm using one set according to the EDF method described in Theorem \ref{normest2}, and estimating the direction using the remaining measurements, as in Corollary \ref{combined}. (Figures 1b, 2b). 

\begin{figure}[h]
\centerline{
\includegraphics[width=90mm]{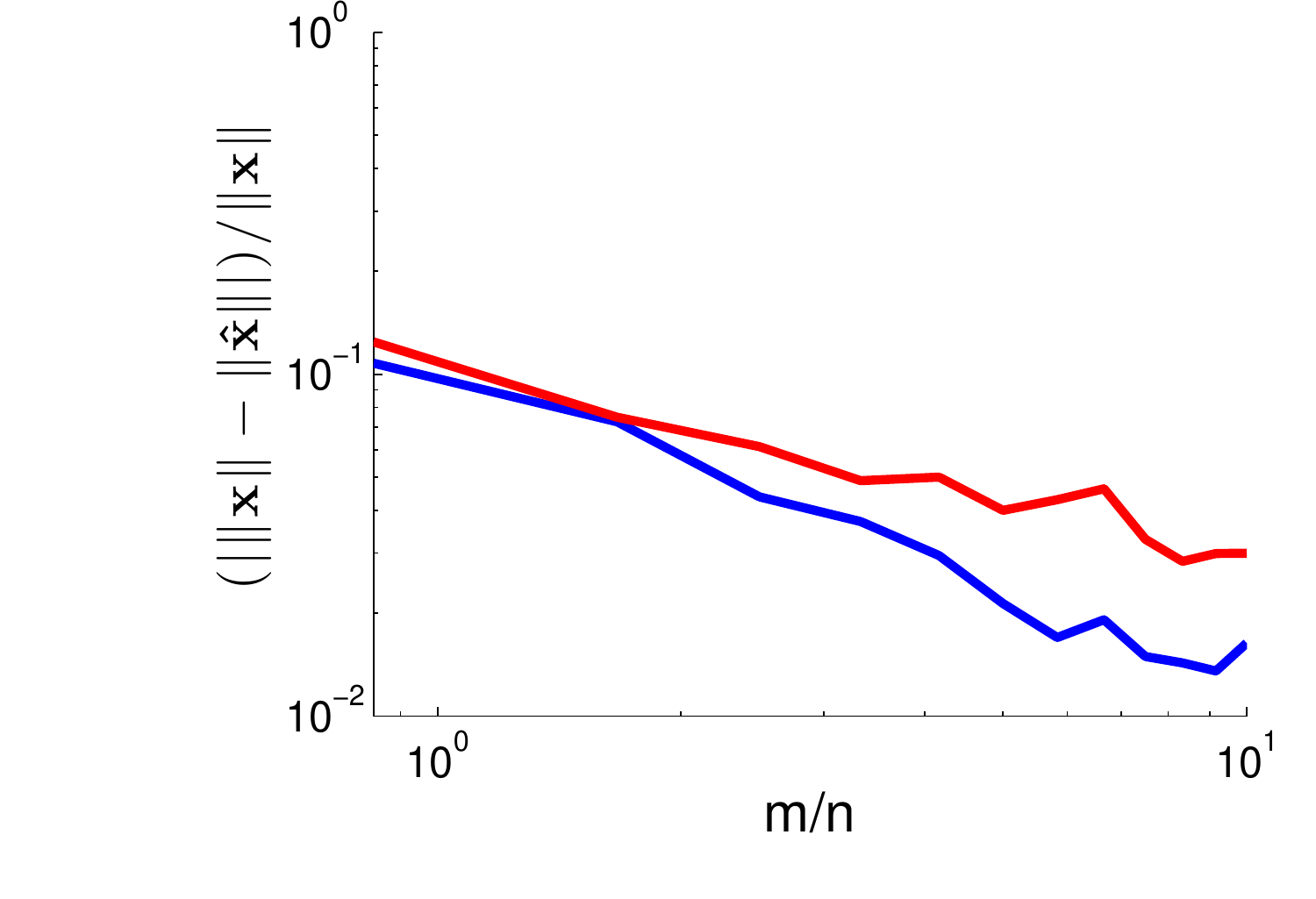}
\includegraphics[width = 90mm]{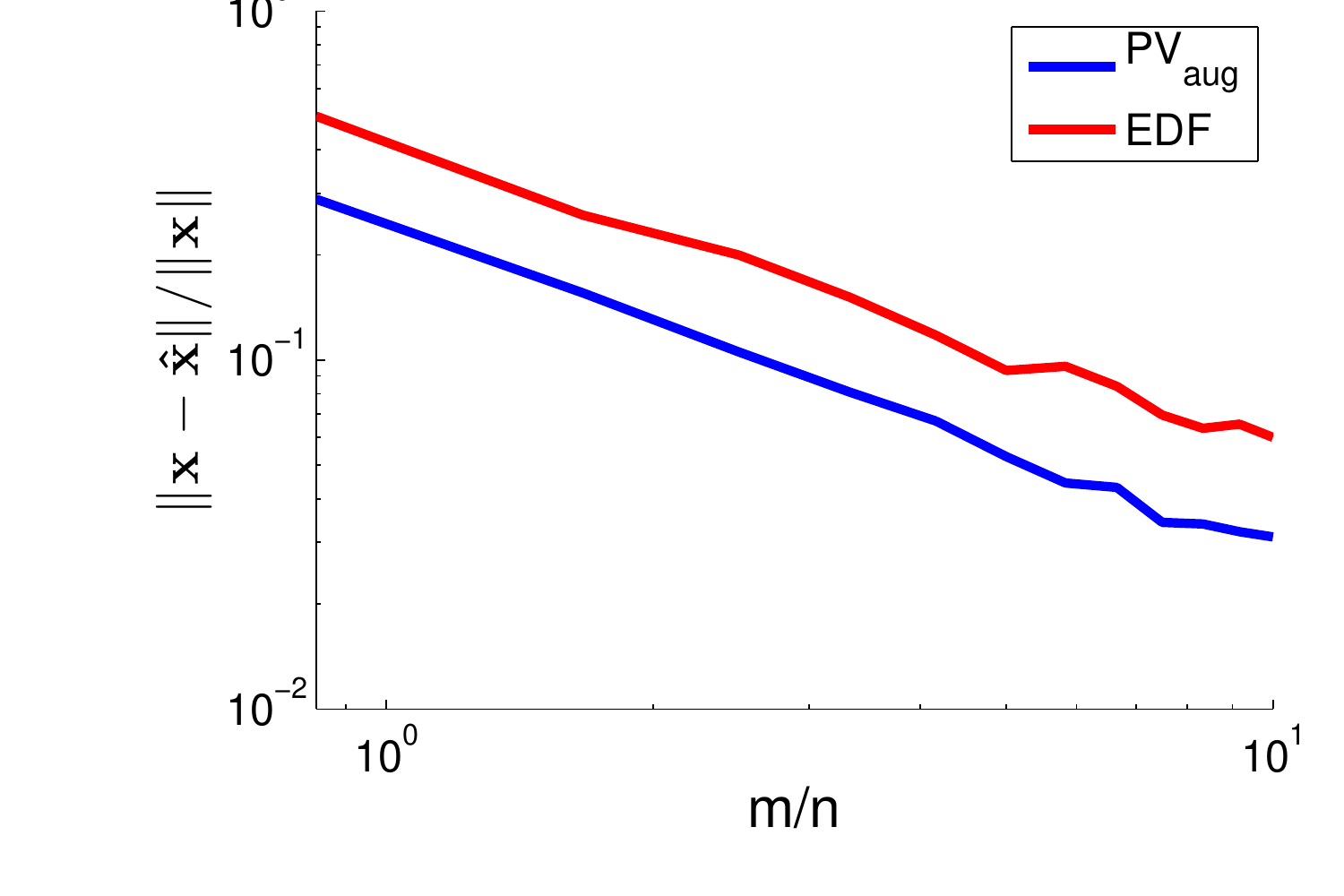}}
 \label{MoverN} 
 \caption{Error of the reconstructed norm (left) and reconstructed signal (right) for values of the number of measurements $m$.  Here, $\tau$ is held at $r$.  Results are averaged over 40 trials.}
\end{figure}

In Figure 1 we plot recovery error for various values of $m/n$.  Note that the \emph{oversampled} regime  $m >n$, while uninteresting in classical compressive sensing setting (with no quantization), is still potentially useful in the one-bit setting, particularly when measurements are fast or cheap relative to finer quantization.  For each value of $m/n$ we report the average error (over  40 trials) in estimating each of  $\vct{x}$ and $\| \vct{x} \|_2$.    The $PV_{aug}$ method outperforms the EDF method in the plotted regime, at the cost of more computation time.

We also explore the effect of the choice of threshold $\tau$ on the accuracy of recovery for both methods.  In the EDF method, each measurement is quantized according to whether it is above or below the same threshold $\tau$.  In the $PV_{aug}$ method, we consider for the same parameter $\tau$  the thresholds $b_i \sim {\cal N}(0,\tau^2)$.  In this case, the expected norm of $\vct{b}$ equals the norm of the threshold vector $\boldsymbol{\tau}=( \tau, \tau,..., \tau)$ used in the EDF method.

\begin{figure}[h]
\centerline{
\includegraphics[width = 90mm]{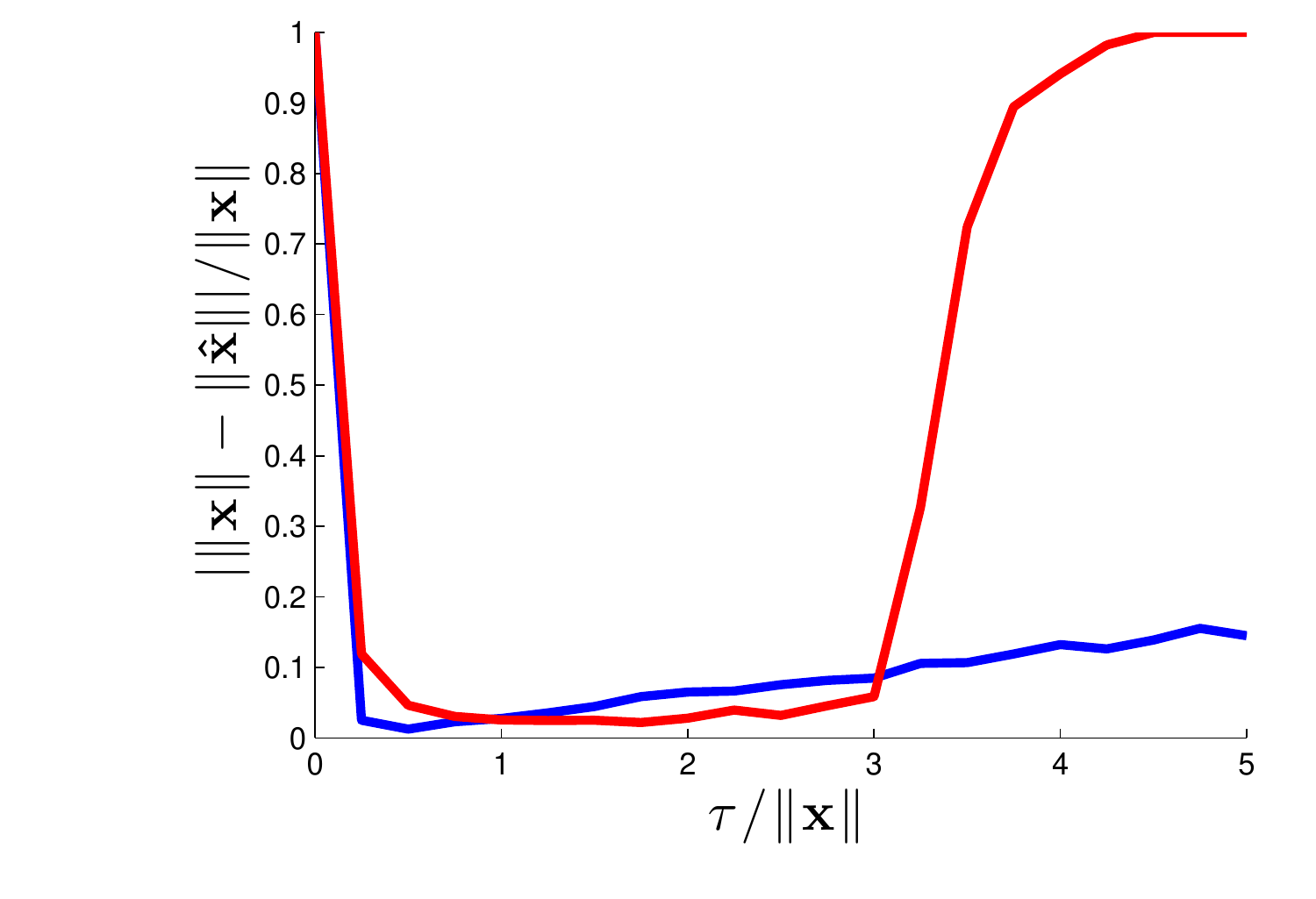}
\includegraphics[width=90mm]{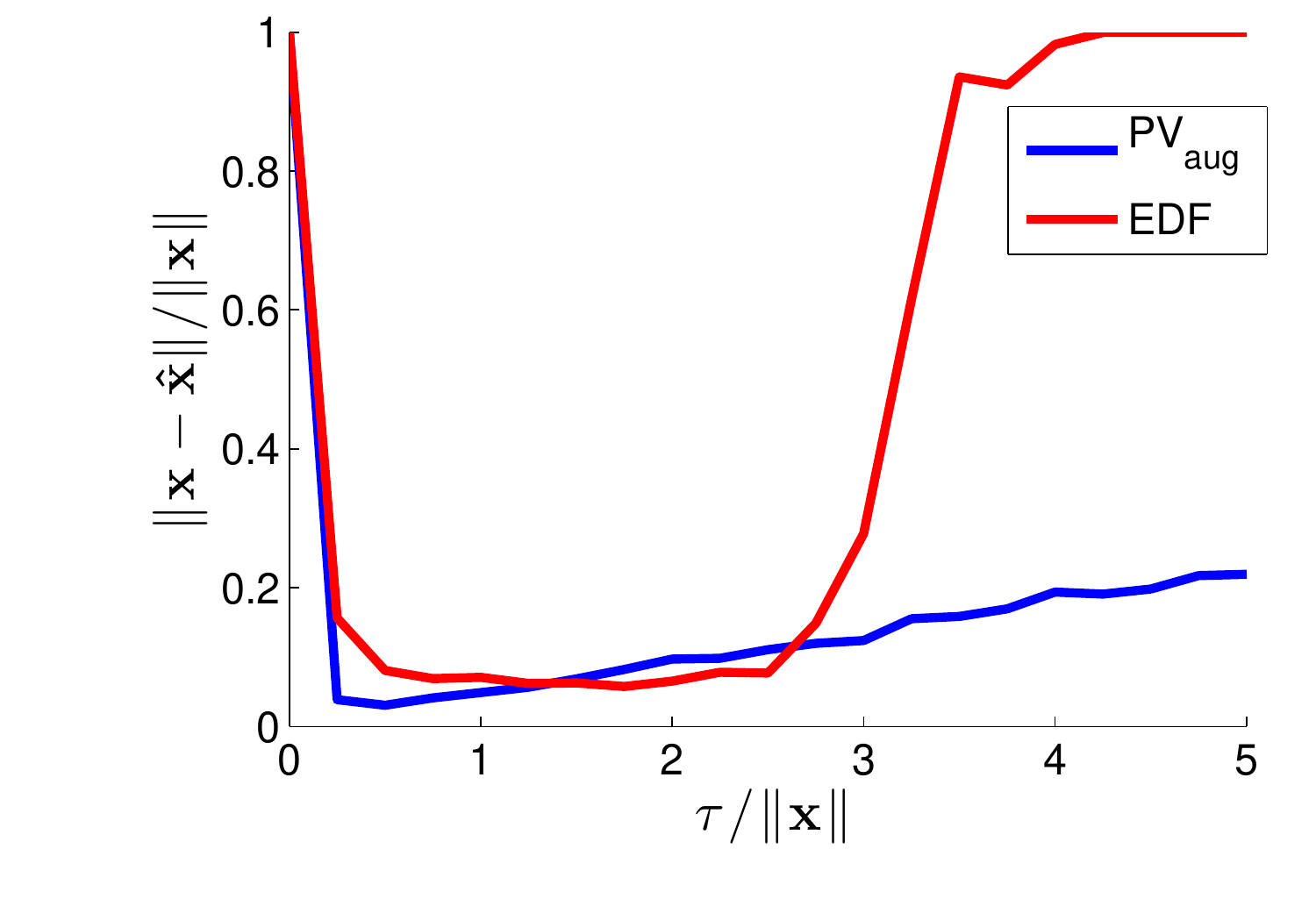}
}
 \label{tauFig} 
 \caption{Error of the reconstructed norm (left) and reconstructed signal (right), and for values of the thresholding parameter $\tau$ relative to $\| \vct{x} \|_2$.  Here m/n=6.  Results are averaged over 40 trials.}
\end{figure}

 We expect reconstruction to be poor when $\tau$ is too large or too small relative to the true norm of $\vct{x}$.   As $\tau$ goes to zero, the proportion of measurements $y _i$ that are $-1$, $\frac{ | \{i :y_i = -1\}|}{m}$ will approach $\frac{1}{2}$ for either fixed $\tau$ or random ${\cal N}(0, \tau^2)$ thresholds.  On the other hand as $\tau$ gets large $\frac{ | \{i :y_i = -1\}|}{m}$ will go to 0 for the deterministically thresholded measurements, and to $\frac{1}{2}$ for thresholds $\sim {\cal N}(0, \tau^2)$.  The poor performance at these two extremes yields the U-shaped error graphs in Figure 2.   We find that the EDF method may slightly outperform the $PV_{\text{aug}}$ method at the optimal choice of $\tau$, but that $PV_{\text{aug}}$  is more robust, its error increasing more gradually as $\tau$ is increased away from $\| \vct{x} \|_2$.

 \section{Conclusions}
We showed that norm recovery, while impossible from one-bit measurements $\sign( \langle \vct{a}_i, \vct{x} \rangle ),$ is indeed possible from one-bit measurements of the form $\sign( \langle \vct{a}_i, \vct{x} \rangle + b_i)$ for known nonzero $b_i$ and for $\vct{a}_i$ with i.i.d. standard Gaussian entries.  
We presented two methods for norm recovery, the first of which also produces an estimate of $\vct{x}$ and uses randomly chosen $b_i$, and the second of which uses fixed, deterministic $b_i$ and produces estimates of $\| \vct{x} \|_2$.  In both cases, we present uniform guarantees of accurate recovery (with high probability) given sufficient a number of measurements, provided we have some prior upper bound (or upper and lower bound) for the norm of $\vct{x}$.

\subsection*{Acknowledgment}
We would like to thank Yaniv Plan for helpful discussions about this topic.  We would also like to thank the anonymous referees for helpful comments and insights which significantly improved the paper.

\bibliographystyle{plain}
\bibliography{arxiv4}

\begin{thebibliography}{10}

\bibitem{DSP}
1-bit compressive sensing.
\newblock {http://dsp.rice.edu/1bitCS/}.
\newblock Accessed: 2015-03-29.

\bibitem{bbr13}
S.~Bahmani, P.~Boufounos, and B.~Raj.
\newblock Robust 1-bit compressive sensing via gradient support pursuit.
\newblock {\em Arxiv preprint}, 2013.

\bibitem{baraniuk2014exponential}
R.~Baraniuk, S.~Foucart, D.~Needell, Y.~Plan, and M.~Wootters.
\newblock Exponential decay of reconstruction error from binary measurements of
  sparse signals.
\newblock {\em arXiv preprint arXiv:1407.8246}, 2014.

\bibitem{benedetto2006sigma}
J.J. Benedetto, A.M. Powell, and {\"O}.~Y{\i}lmaz.
\newblock {Sigma-delta ($\Sigma\Delta$) quantization and finite frames}.
\newblock {\em Information Theory, IEEE Transactions on}, 52(5):1990--2005,
  2006.

\bibitem{blum:sdf}
J.~Blum, M.~Lammers, A.M. Powell, and {\"O}.~Y{\i}lmaz.
\newblock {Sobolev duals in frame theory and sigma-delta quantization}.
\newblock {\em Journal of Fourier Analysis and Applications}, 16(3):365--381,
  2010.

\bibitem{Boufounos09}
P.~Boufounos.
\newblock Greedy sparse signal reconstruction from sign measurements.
\newblock In {\em Signals, Systems and Computers, 2009 Conference Record of the
  Forty-Third Asilomar Conference on}, pages 1305--1309. IEEE, 2009.

\bibitem{Boufounos08}
P.~Boufounos and R.~Baraniuk.
\newblock 1-bit compressive sensing.
\newblock {\em 42nd Annual Conference on Information Sciences and Systems
  (CISS)}, 2008.

\bibitem{boufounos2014quantization}
P.~Boufounos, L.~Jacques, F.~Krahmer, and R.~Saab.
\newblock Quantization and compressive sensing.
\newblock {\em arXiv preprint arXiv:1405.1194}, 2014.

\bibitem{carota06}
E.~{C}and{\`e}s, , {T}. {T}ao, and {J}. {R}omberg.
\newblock {R}obust uncertainty principles: exact signal reconstruction from
  highly incomplete frequency information.
\newblock {\em {I}{E}{E}{E} {T}rans. {I}nform. {T}heory}, 52(2):489--509, 2006.

\bibitem{candes2006compressive}
E.~Cand{\`e}s.
\newblock Compressive sampling.
\newblock In {\em Proceedings oh the International Congress of Mathematicians:
  Madrid, August 22-30, 2006: {I}nvited lectures}, pages 1433--1452, 2006.

\bibitem{CRT05}
E.~Cand{\`e}s, J.~Romberg, and T.~Tao.
\newblock Stable signal recovery from incomplete and inaccurate measurements.
\newblock {\em Communications on Pure and Applied Mathematics}, LIX:1207?1223,
  2006.

\bibitem{C13}
M.~Cheraghchi.
\newblock Improved constructions for non-adaptive threshold group testing.
\newblock {\em Algorithmica}, 67(3):384--417, 2013.

\bibitem{dk06}
O.~Dabeer and A~Karnik.
\newblock Signal parameter estimation using 1-bit dithered quantization.
\newblock {\em IEEE Trans. on Infor. Theory}, 52(12):5389--5405, 2006.

\bibitem{Dasgupta}
S.~Dasgupta and A.~Gupta.
\newblock An elementary proof of a theorem of {J}ohnson and {L}indenstrauss.
\newblock {\em Random Structures and Algorithms}, 22:60--65, 2003.

\bibitem{Daub-dev}
I.~Daubechies and R.~DeVore.
\newblock {Approximating a bandlimited function using very coarsely quantized
  data: a family of stable sigma-delta modulators of arbitrary order}.
\newblock {\em Ann. Math.}, 158(2):679--710, 2003.

\bibitem{davenport2012}
M.~Davenport, Y.~Plan, E.~Berg, and M.~Wootters.
\newblock 1-bit matrix completion.
\newblock {\em Information and Inference}, 3(3):189--223, 2014.

\bibitem{DGK10}
P.~{D}eift, C.~S. G{\"u}nt{\"u}rk, and F.~{K}rahmer.
\newblock An optimal family of exponentially accurate one-bit sigma-delta
  quantization schemes.
\newblock {\em Communications on Pure and Applied Mathematics}, 64(7):883--919,
  2011.

\bibitem{donoho2006compressed}
D.~Donoho.
\newblock Compressed sensing.
\newblock {\em Information Theory, IEEE Transactions on}, 52(4):1289--1306,
  2006.

\bibitem{Dvoretzky1956}
A.~Dvoretzky, J.~Kiefer, and J.~Wolfowitz.
\newblock Asymptotic minimax character of the sample distribution function and
  of the classical multinomial estimator.
\newblock {\em The Annals of Mathematical Statistics}, pages 642--669, 1956.

\bibitem{foucart2013mathematical}
S.~Foucart and H.~Rauhut.
\newblock {\em A mathematical introduction to compressive sensing}.
\newblock Springer, 2013.

\bibitem{G-exp}
C.S. G{\"u}nt{\"u}rk.
\newblock One-bit sigma-delta quantization with exponential accuracy.
\newblock {\em Communications on Pure and Applied Mathematics},
  56(11):1608--1630, 2003.

\bibitem{Jacques11}
L.~Jacques, J.~Laska, P.~Boufounos, and R.~Baraniuk.
\newblock Robust 1-bit compressive sensing via binary stable embeddings of
  sparse vectors.
\newblock {\em arXiv preprint arXiv:1104.3160}, 2011.

\bibitem{jacques2011}
L.~Jacques, J.~Laska, P.~Boufounos, and R.~Baraniuk.
\newblock Robust 1-bit compressive sensing via binary stable embeddings of
  sparse vectors.
\newblock {\em Information Theory, IEEE Transactions on}, 59(4):2082--2102,
  April 2013.

\bibitem{JL}
{W}.~{B}. {J}ohnson and {J} {L}indenstrauss.
\newblock {E}xtensions of {L}ipschitz mappings into a {H}ilbert space.
\newblock {\em {C}ontemp. {M}ath}, 26:189--206, 1984.

\bibitem{Kamilov}
U.~Kamilov, A.~Bourquard, A.~Amini, and M.~Unser.
\newblock One-bit measurements with adaptive thresholds.
\newblock {\em Signal Processing Letters, IEEE}, 19(10):607--610, 2012.

\bibitem{KSW12}
F.~Krahmer, R.~Saab, and R.~Ward.
\newblock Root-exponential accuracy for coarse quantization of finite frame
  expansions.
\newblock {\em Information Theory, IEEE Transactions on}, 58(2):1069 --1079,
  February 2012.

\bibitem{KSY13}
F.~Krahmer, R.~Saab, and {\"O}~Y{\i}lmaz.
\newblock Sigma-delta quantization of sub-{G}aussian frame expansions and its
  application to compressed sensing.
\newblock {\em Information and Inference}, page iat007, 2014.

\bibitem{laska2012regime}
J.~Laska and R.~Baraniuk.
\newblock Regime change: Bit-depth versus measurement-rate in compressive
  sensing.
\newblock {\em Signal Processing, IEEE Transactions on}, 60(7):3496--3505,
  2012.

\bibitem{plan2013one}
Y.~Plan and R.~Vershynin.
\newblock One-bit compressed sensing by linear programming.
\newblock {\em Communications on Pure and Applied Mathematics}, 2013.

\bibitem{Plan12}
Y.~Plan and R.~Vershynin.
\newblock Robust 1-bit compressed sensing and sparse logistic regression: A
  convex programming approach.
\newblock {\em IEEE Trans. Infor. Theory}, 59(1):482--494, 2013.

\bibitem{rac03}
D.~Rousseau, G.~Anand, and F~Chapeau-Blondeau.
\newblock Nonlinear estimation from quantized signals: Quantizer optimization
  and stochastic resonance.
\newblock {\em Proc. 3rd Int. Symp. Physics in Signal and Image Processing},
  pages 89--92, 2003.

\bibitem{saab2015quantization}
R.~Saab, R.~Wang, and O.~Yilmaz.
\newblock Quantization of compressive samples with stable and robust recovery.
\newblock {\em arXiv preprint arXiv:1504.00087}, 2015.

\bibitem{Yan12}
M.~Yan, Y.~Yang, and S.~Osher.
\newblock Robust 1-bit compressive sensing using adaptive outlier pursuit.
\newblock {\em Signal Processing, IEEE Transactions on}, 60(7):3868--3875,
  2012.

\end{thebibliography}

\end{document}